\documentclass{colt2013} 

\usepackage{algorithm,algorithmic}
\usepackage{bm, multirow, paralist}
\usepackage{cite}

\newtheorem{thm}{Theorem}

\newtheorem{cor}[thm]{Corollary}

\def \R {\mathbb{R}}

\def \y {\mathbf{y}}

\def \E {\mathrm{E}}
\def \x {\mathbf{x}}
\def \a {\mathbf{a}}
\def \L {\mathcal{L}}

\def \v {\mathbf{v}}

\def \z {\mathbf{z}}

\def \xh {\widehat{\x}}

\def \u {\mathbf{u}}

\def \v {\mathbf{v}}
\def \w {\mathbf{w}}

\def \R {\mathbb{R}}

\def \A {\mathcal{A}}

\def \Hh {\widehat{H}}

\def \tr {\mbox{tr}}

\def \N {\mathcal{N}}
\def \wh {\widehat{\w}}
\def \abh {\widehat{\bm \alpha}}
\def \ah {\widehat{\a}}
\def \ab {\bm \alpha}
\def \gb {\bm \gamma}
\def \gbt {\widetilde{\gb}}
\def \eb {\bm \eta}
\def \wt {\widetilde{\w}}
\def \Gh {\widehat{G}}

\makeatletter
\AtBeginDocument{\Hy@breaklinkstrue}
\makeatother

\title[Dual Random Projection]{Recovering the Optimal Solution by \\ Dual Random Projection}

   \coltauthor{\Name{Lijun Zhang} \Email{zhanglij@msu.edu} \\
   \Name{Mehrdad Mahdavi} \Email{mahdavim@msu.edu}\\
    \Name{Rong Jin} \Email{rongjin@cse.msu.edu}\\
   \addr Department of Computer Science and Engineering\\
Michigan State University, East Lansing, MI 48824, USA\\
\Name{Tianbao Yang} \Email{tyang@ge.com}\\
\addr GE Global Research, San Ramon, CA 94583, USA\\
\Name{Shenghuo Zhu} \Email{zsh@nec-labs.com}\\
\addr NEC Laboratories America, Cupertino, CA 95014, USA}

\begin{document}
\maketitle
\begin{abstract}
Random projection has been widely used in data classification. It maps high-dimensional data into a low-dimensional subspace in order to reduce the computational cost in solving the related optimization problem. While previous studies are focused on analyzing the classification performance of using random projection, in this work, we consider the {\it recovery problem}, i.e., how to accurately recover the optimal solution to the original optimization problem in the high-dimensional space based on the solution learned from the subspace spanned by random projections. We present a simple algorithm, termed {\it Dual Random Projection}, that uses the dual solution of the low-dimensional optimization problem to recover the optimal solution to the original problem. Our theoretical analysis shows that with a high probability, the proposed algorithm is able to accurately recover the optimal solution to the original problem, provided that the data matrix is of low rank or can be well approximated by a low rank matrix.
\end{abstract}

\begin{keywords}
Random projection, Primal solution, Dual solution, Low rank
\end{keywords}
\section{Introduction}
Random projection is a simple yet powerful dimensionality reduction technique that projects the original high-dimensional
data onto a low-dimensional subspace using a random matrix~\citep{kaski-1998-dimensionality,bingham-2001-random}. It has been successfully applied to many machine learning tasks, including classification~\citep{Fradkin-2003-experiment,Vempala-2004-random,rahimi-2008-random}, regression~\citep{maillard-2012-linear}, clustering~\citep{fern-2003-random,boutsidis-2010-random},  manifold learning~\citep{dasgupta:2008:tree,Freund-2008-Learning}, and information retrieval~\citep{Goal-2005-face}.

In this work, we focus on random projection for classification. While previous studies were devoted to analyzing the classification performance using random projection~\citep{arriage-1999-algorithmic,ML06:Balcan,RP:SVM,Margin_RP}, we examine the effect of random projection from a very different aspect. In particular, we are interested in accurately recovering the optimal solution to the original high-dimensional optimization problem using random projection. This is particularly useful for feature selection~\citep{guyon-2003-introduction}, where important features are often selected based on their weights in the linear prediction model learned from the training data.
In order to ensure that similar features are selected, the prediction model based on random projection needs to be close to the model obtained by solving the original optimization problem directly.

The proposed algorithm for recovering the optimal solution consists of two simple steps. In the first step, similar to previous studies, we apply random projection to reducing the dimensionality of the data, and then solve a low-dimensional optimization problem. In the second step, we construct the dual solution of the low-dimensional problem from its primal solution, and then use it to recover the optimal solution to the original high-dimensional problem. Our analysis reveals that with a high probability, we are able to recover the optimal solution with a small error by using $\Omega(r\log r)$ projections, where $r$ is the rank of the data matrix. A similar result also holds when the data matrix can be well approximated by a low rank matrix. We further show that the proposed algorithm can be applied iteratively to recovering the optimal solution with a relative error $\epsilon$ by using $O(\log 1/\epsilon)$ iterations.

The rest of the paper is arranged as follows. Section~\ref{sec:problem} describes the problem of recovering optimal solution by random projection, the theme of this work. Section~\ref{sec:algorithm} describes the dual random projection approach for recovering the optimal solution. Section~\ref{sec:results} presents the main theoretical results for the proposed algorithm. Section~\ref{sec:analysis} presents the proof for the theorems stated in Section~\ref{sec:results}. Section~\ref{sec:conclusion} concludes this work.
\section{The Problem} \label{sec:problem}
Let $(\x_i, y_i), i=1, \ldots, n$ be a set of training examples, where $\x_i \in \R^d$ is a vector of $d$ dimension and $y_i \in \{-1, +1\}$ is the binary class assignment for $\x_i$. Let $X = (\x_1, \ldots, \x_n)$ and $\y = (y_1, \ldots, y_n)^{\top}$ include input patterns and the class assignments of all training examples. A classifier $\w \in \R^d$ is learned by solving the following optimization problem:
\begin{eqnarray} \label{eqn:primal}
    \min\limits_{\w \in \R^d} \frac{\lambda}{2}\|\w\|^2 + \sum_{i=1}^n \ell(y_i\x_i^{\top}\w),
\end{eqnarray}
where $\ell(z)$ is a convex loss function that is differentiable\footnote{For non-differentiable loss functions such as hinge loss, we could apply the smoothing technique described in~\citep{Nesterov:2005:SMN} to make it differentiable.}. By writing $\ell(z)$ in its convex conjugate form, i.e.,
\[
    \ell(z) = \max\limits_{\alpha \in \Omega} \alpha z - \ell_*(\alpha),
\]
where $\ell_*(\alpha)$ is the convex conjugate of $\ell(z)$ and $\Omega$ is the domain of the dual variable, we get the dual optimization problem:
\begin{eqnarray}
    \max\limits_{\ab \in \Omega^n} -\sum_{i=1}^n \ell_*(\alpha_i) - \frac{1}{2\lambda} \ab^{\top} G \ab, \label{eqn:dual}
\end{eqnarray}
where $\ab=(\alpha_1,\cdots, \alpha_n)^{\top}$ and $D(\y) = \mbox{diag}(\y)$ and $G$ is the Gram matrix given by
\begin{eqnarray}\label{eqn:G}
    G = D(\y)X^{\top}XD(\y).
\end{eqnarray}
In the following, we denote by $\w_* \in \R^d$ the optimal primal solution to~(\ref{eqn:primal}), and by $\ab_* \in \R^n$  the optimal dual solution to~(\ref{eqn:dual}).  The following proposition connects $\w_*$ and $\ab_*$.
\begin{proposition}\label{prop:1} Let $\w_* \in \R^d$  be the optimal primal solution to~(\ref{eqn:primal}), and $\ab_* \in \R^n$  be the optimal dual solution to~(\ref{eqn:dual}). We have
\begin{eqnarray}
\w_*=  -\frac{1}{\lambda}XD(\y)\ab_*,\textrm{ and }[\ab_*]_i= \nabla\ell\left(y_i\x_i^{\top}\w_*\right), \ i=1,\ldots, n.
\end{eqnarray}
\end{proposition}
The proof of Proposition~\ref{prop:1} and other omitted proofs are deferred to the Appendix.  When the dimensionality $d$ is high and the number of training examples $n$ is large, solving either the primal problem in (\ref{eqn:primal}) or the dual problem in (\ref{eqn:dual}) can be computationally expensive. To reduce the computational cost, one common approach is to significantly reduce the dimensionality by random projection. Let $R \in \R^{d\times m}$ be a Gaussian random matrix, where each entry $R_{i,j}$ is independently drawn from a Gaussian distribution $\N(0, 1)$ and $m$ is significantly smaller than $d$. Using the random matrix $R$, we generate a low-dimensional representation for each input example by
\begin{eqnarray}
    \xh_i = \frac{1}{\sqrt{m}} R^{\top} \x_i, \label{eqn:proj-data}
\end{eqnarray}
and solve the following low-dimensional optimization problem:
\begin{eqnarray}
    \min\limits_{\z \in \R^m} \frac{\lambda}{2}\|\z\|^2 + \sum_{i=1}^n \ell(y_i\z^{\top}\xh_i). \label{eqn:primal-1}
\end{eqnarray}
The corresponding dual problem is written as
\begin{eqnarray}
    \max\limits_{\ab \in \Omega^n} -\sum_{i=1}^n \ell_*(\alpha_i) - \frac{1}{2\lambda} \ab^{\top} \Gh \ab, \label{eqn:dual-1}
\end{eqnarray}
where
\begin{eqnarray}\label{eqn:Gh}
    \Gh = D(\y)X^{\top} \frac{RR^{\top}}{m} X D(\y).
\end{eqnarray}
Intuitively, the choice of Gaussian random matrix $R$ is justified by the expectation of the dot-product between any two examples in the projected space is equal to the dot-product in the original space, i.e.,
\[
\E[\xh_i^{\top}\xh_j]= \x_i^{\top}\E\left[\frac{1}{m}RR^{\top}\right]\x_j = \x_i^{\top}\x_j,
\]
where the last equality follows from $\E[RR^{\top}/m]=I$. Thus, $G=\Gh$ holds in expectation.

Let $\z_* \in \R^m$ denote the optimal primal solution to the low-dimensional problem~(\ref{eqn:primal-1}), and $\abh_* \in \R^n$ denote the optimal dual solution to~(\ref{eqn:dual-1}).  Similar to Proposition~\ref{prop:1}, the following proposition connects $\z_*$ and $\abh_*$.
\begin{proposition}\label{prop:2} We have
\begin{eqnarray} \label{eqn:relation:2}
\z_*=  -\frac{1}{\lambda}\frac{1}{\sqrt{m}}R^{\top}XD(\y)\abh_*,\textrm{ and } [\abh_*]_i = \nabla\ell
\left(\frac{y_i}{\sqrt{m}}\x_i^{\top}R\z_*\right), \ i=1,\ldots, n.
\end{eqnarray}
\end{proposition}
Given the optimal solution $\z_*\in\R^m$, the data point $\x\in\mathbb R^d$ is classified by $\x^{\top}R\z_*/\sqrt{m}$, which is equivalent to defining a new solution $\wh\in\R^d$ given below, which we refer to as the naive solution
\begin{eqnarray}
    \wh = \frac{1}{\sqrt{m}} R\z_*. \label{eqn:wh}
\end{eqnarray}
The classification performance of $\wh$ has been examined by many studies~\citep{arriage-1999-algorithmic,ML06:Balcan,RP:SVM,Margin_RP}. The general conclusion is that when the original data is linearly separable with a large margin, the classification error for $\wh$ is usually small.

Although these studies show that $\wh$ can achieve a small classification error under appropriate assumptions, it is unclear whether $\wh$ is a good approximation of the optimal solution $\w_*$. In fact, as we will see in Section~\ref{sec:results}, the naive solution is almost guaranteed to be a BAD approximation of the optimal solution, that is, $\|\wh - \w_*\|_2 = \Omega(\sqrt{d/m}\|\w_*\|_2)$. This observation leads to an interesting question: \textit{is it possible to accurately recover the optimal solution $\w_*$ based on $\z_*$, the optimal solution to the low-dimensional optimization problem}?


\paragraph{Relationship to Compressive Sensing} The proposed problem is closely related to compressive sensing~\citep{donoho-2006-cs,candes-2008-introduction} where the goal is to recover a high-dimensional but sparse vector using a small number of random measurements. The key difference between our work and compressive sensing is that we do not have the direct access to the random measurement of the target vector (which in our case is $\w_* \in \R^d$). Instead, $\z_* \in \R^m$ is the optimal solution to (\ref{eqn:primal-1}), the primal problem using random projection. However, the following theorem shows that $\z_*$ is a good approximation of $R^{\top}\w_*/\sqrt{m}$, which includes $m$ random measurements of $\w_*$, if the data matrix $X$ is of low rank and the number of random measurements $m$ is sufficiently large.
\begin{thm} \label{thm:rand-measure}
For any $0 < \varepsilon \leq 1/2$, with a probability at least $1 - \delta - \exp(-m/32)$, we have
\[
    \|\sqrt{m}\z_* - R^{\top}\w_*\|_2 \leq \frac{\sqrt{2}\varepsilon}{\sqrt{1 - \varepsilon}}\|R^{\top}\w_*\|_2,
\]
provided
\[
m \geq \frac{(r+1) \log(2 r/\delta)}{c \varepsilon ^2},
\]
where constant $c$ is at least $1/4$, and $r$ is the rank of $X$.
\end{thm}
Given the approximation bound in Theorem~\ref{thm:rand-measure}, it is appealing to reconstruct $\w_*$ using the compressive sensing algorithm provided that $\w_*$ is sparse to certain bases. We note that the low rank assumption for data matrix $X$ implies that $\w_*$ is sparse with respect to the singular vectors of $X$. However, since $\z_*$ only provides an approximation to the random measurements of $\w_*$, running the compressive sensing algorithm will not be able to perfectly recover $\w_*$ from $\z_*$. In Section~\ref{sec:algorithm}, we present an algorithm, that recovers $\w_*$ with a small error, provided that the data matrix $X$ is of low rank. Compared to the compressive sensing algorithm, the main advantage of the proposed algorithm is its computational simplicity; it neither computes the singular vectors of $X$ nor solves an optimization problem that minimizes the $\ell_1$ norm.
\begin{algorithm}[t]
\caption{A Dual Random Projection Approach for Recovering Optimal Solution}
\begin{algorithmic}[1]
\STATE {\bf Input:} input patterns $X \in \R^{d\times n}$, binary class assignment $\y \in \{-1, +1\}^n$, and sample size $m$
\STATE Sample a Gaussian random matrix $R \in \R^{d\times m}$ and compute $\widehat{X}=[\xh_1, \ldots, \xh_n] = R^{\top}X/\sqrt{m}$
\STATE Obtain the primal solution $\z_*  \in \R^m$ by solving the optimization problem in~(\ref{eqn:primal-1})
\STATE Construct the dual solution $\abh_* \in \R^n$ by Proposition~\ref{prop:2}, i.e.,
\[
 [\abh_*]_i = \nabla\ell
\left(\frac{y_i}{\sqrt{m}}\x_i^{\top}R\z_*\right), \ i=1,\ldots, n
\]
\STATE Compute $\wt \in \R^d$ according to~(\ref{eqn:wt}), i.e., $\wt = - XD(\y)\abh_*/\lambda$
\STATE {\bf Output:} the recovered solution $\wt$
\end{algorithmic}
\label{alg:1}
\end{algorithm}

\section{Algorithm} \label{sec:algorithm}
To motivate our algorithm, let us revisit the optimal primal solution $\w_*$ to~(\ref{eqn:primal}), which is given in Proposition~\ref{prop:1}, i.e.,
\begin{equation}\label{eqn:pw}
\w_*=  -\frac{1}{\lambda}XD(\y)\ab_*,
\end{equation}
where $\ab_*$ is the optimal solution to the dual problem~(\ref{eqn:dual}).  Given the projected data $\xh=R^{\top}\x/\sqrt{m}$, we have reached an approximate dual problem in~(\ref{eqn:dual-1}). Comparing it with the dual problem in~(\ref{eqn:dual}), the only difference is that the Gram matrix $G = D(\y)X^{\top}XD(\y)$ in~(\ref{eqn:dual}) is replaced with $\Gh = D(\y)X^{\top} RR^{\top} X D(\y)/m$ in~(\ref{eqn:dual-1}). Recall that $\E[RR^{\top}/m]=I$. Thus, when the number of random projections $m$ is sufficiently large, $\Gh$ will be close to the $G$ and we would also expect $\abh_*$ to be close to $\ab_*$. As a result, we can use $\abh_*$ to approximate $\ab_*$ in~(\ref{eqn:pw}), which yields a recovered prediction model given below:
\begin{eqnarray}
    \wt = -  \frac{1}{\lambda}XD(\y)\abh_* =  -\sum_{i=1}^n \frac{1}{\lambda}y_i [\abh_*]_i \x_i. \label{eqn:wt}
\end{eqnarray}

Note that the key difference between the recovered solution $\wt$ and the naive solution $\wh$ is that $\wh$ is computed by mapping the optimal primal solution $\z_*$ in the projected space back to the original space via the random matrix $R$, while $\wt$ is computed directly in the original space using the approximate dual solution $\abh_*$. Therefore, the naive solution  $\wh$  lies in the subspace spanned by the column vectors in the random matrix $R$ (denoted by $\A_R$), while the recovered solution $\wt$ lies in the subspace that also contains the optimal solution $\w_*$, i,e., the subspace spanned by columns of  $X$ (denoted by $\A$). The mismatch between spaces $\A_R$ and $\A$ leads to the large approximation error for $\wh$.

According to Proposition~\ref{prop:2}, we can construct the dual solution $\abh_*$ from the primal solution $\z_*$. Thus, we do not need to solve the dual problem in~(\ref{eqn:dual-1}) to obtain $\abh_*$. Instead, we solve the low-dimensional optimization problem in (\ref{eqn:primal-1}) to get $\z_*$ and construct $\abh_*$ from it. Algorithm~\ref{alg:1} shows the details of the proposed method. We note that although dual variables have been widely used in the analysis of convex optimization~\citep{boyd-2004-convex,hazan-2011-beating} and online learning~\citep{shai-2006-online}, to the best of our knowledge, this is the first time that dual variables are used in conjunction with random projection for recovering the optimal solution.

\begin{algorithm}[t]
\caption{An Iterative Dual Random Projection Approach for Recovering Optimal Solution}
\begin{algorithmic}[1]
\STATE {\bf Input:} input patterns $X \in \R^{d\times n}$, binary class assignment $\y \in \{-1, +1\}^n$, sample size $m$, and number of iterations $T$
\STATE Sample a Gaussian random matrix $R \in \R^{d\times m}$ and compute $\widehat X = R^{\top}X/\sqrt{m}$
\STATE Initialize $\wt^0 = \mathbf{0}$
\FOR{$t = 1, \ldots, T$}
    \STATE Obtain $\z^t_*  \in \R^m$ by solving the following optimization problem
    \begin{eqnarray}
        \min\limits_{\z\in\R^m} \frac{\lambda}{2}\left\|\z + \frac{1}{\sqrt{m}} R^{\top} \wt^{t-1}\right\|_2^2 + \sum_{i=1}^n \ell\left(y_i\z^{\top}\xh_i + y_i [\wt^{t-1}]^{\top}\x_i\right) \label{eqn:primal-a}
    \end{eqnarray}
    \STATE Construct the dual solution $\abh^{t}_* \in \R^n$ using
    \[
   [\abh^t_*]_{i}=\nabla\ell\left(y_i \xh_i^{\top}\z^t_* +  y_i [\wt^{t-1}]^{\top} \x_i\right), \ i=1,\ldots, n
   \]
    \STATE Update the solution by $\wt^{t} =  - XD(\y)\abh^{t}_*/\lambda$
\ENDFOR
\STATE {\bf Output} the recovered solution $\wt^T$
\end{algorithmic}
\label{alg:2}
\end{algorithm}

To further reduce the recovery error, we develop an iterative method shown in Algorithm~\ref{alg:2}.  The idea comes from that fact that if $\|\wt-\w_*\|_2 \leq \epsilon \|\w_*\|_2$ with a small $\epsilon$,  we can apply the same dual random projection algorithm again to recover $\Delta\w = \w_* - \wt$, which will result in a recovery error of $\epsilon\|\Delta \w\|_2 \leq \epsilon^2\|\w_*\|_2$. If we repeat the above process with $T$ iterations, we should be able to obtain a solution with a recovery error of $\epsilon^T \|\w_*\|_2$.  This simple intuition leads to an iterative method shown in Algorithm~\ref{alg:2}. At the $t$-th iteration, given the recovered solution $\wt^{t-1}$ obtained from the previous iteration, we solve the optimization problem in (\ref{eqn:primal-a}) that is designed to recover $\w_* - \wt^{t-1}$. The detailed derivation of Algorithm~\ref{alg:2} is provided in Section~\ref{sec:iterative}.

It is important to note that although Algorithm~\ref{alg:2} consists of multiple iterations, the random projection of the data matrix is only computed once before the start of the iterations. This important feature makes the iterative algorithm computationally attractive as calculating random projections of a large data matrix is computationally expensive and has been the subject of many studies, e.g.,~\citep{achlioptas-2003-database,Liberty08densefast,2010arXiv1011.2590B}. However, it is worth noting that at each iteration in Algorithm~\ref{alg:2}, we need to compute the dot-product $[\wt^{t-1}]^{\top}\x_i$ for all training data in the original space. We also note that Algorithm~\ref{alg:2} is related to the Epoch gradient descent algorithm~\citep{hazan-20110-beyond} for stochastic optimization in the sense that the solution obtained from the previous iteration serves as the starting point to the optimization problem at the current iteration. Unlike the algorithm in~\citep{hazan-20110-beyond}, we do not shrink the domain size over the iterations in Algorithm~\ref{alg:2}.

\paragraph{Application to the Square Loss}
In the following, we take the square loss $\ell(z)=\frac{1}{2}(1-z)^2$ as an example to illustrate the recovery procedure. The original optimization problem, which is refereed to as the ridge regression~\citep{The-Elements-of-Statistical-Learning-2009}, is
\[
    \min\limits_{\w \in \R^d} \frac{\lambda}{2}\|\w\|^2 + \frac{1}{2} \sum_{i=1}^n (1-y_i\x_i^{\top}\w)^2 \overset{y_i \in \{\pm 1\}}{=}\frac{\lambda}{2}\|\w\|^2 + \frac{1}{2} \sum_{i=1}^n (y_i-\x_i^{\top}\w)^2 .
\]
Setting the derivative to $\mathbf{0}$, we obtain the optimal solution
\begin{eqnarray}
\w_*&=&\left( \lambda I +XX^\top \right)^{-1} X \y = X\left( \lambda I +X^\top X \right)^{-1}  \y.\label{eqn:w_*:least:2}
\end{eqnarray}
where the last equality follows from the Woodbury matrix identity~\citep{Matrix-computations}. Thus, the computational cost is either $O(d^2n +d^3)$ or $O(n^2d +n^3)$.

Following the dual random projection algorithm, we first solve the the low-dimensional problem in (\ref{eqn:primal-1}), whose solution is $\z_*=( \lambda I + \widehat{X}\widehat{X}^\top )^{-1} \widehat{X} \y$, where $\widehat{X}= R^{\top}  X/\sqrt{m}$. Then, we construct the dual solution  $\abh_*  =   D(\y) \widehat{X}^{\top} \z_*- \mathbf{1}$. Finally, we recover the optimal solution $\wt  =  -  \frac{1}{\lambda}XD(\y)\abh_* $. It is straightforward to check that computational cost of our algorithm is $O(mnd+m^2n +m^3)$, which is significantly smaller than that of (\ref{eqn:w_*:least:2}) when both $d$ and $n$ are large.

After some algebraic manipulation, we can show that
\begin{equation} \label{eqn:wt:least:2}
    \wt = X \left( \lambda I +  X^\top \frac{S S^\top }{m} X \right)^{-1} \y.
\end{equation}
Comparing (\ref{eqn:w_*:least:2}) with (\ref{eqn:wt:least:2}), we can see the difference between $\w_*$ and $\wt$ comes from the Gram matrix. When $m$ is large enough, $X^\top R R^\top X /m$ is close to the $X^\top X$, and as a result $\wt$ is also close to $\w_*$.
\section{Main Results} \label{sec:results}
In this section, we will bound the recovery error $\|\w_*-\wt\|_2$ of  dual random projection. We first assume $X$ is of low rank, and then extend the results to the full rank case.
\subsection{Low Rank}
The low rank assumption is closely related to the sparsity assumption made in compressive sensing. This is because $\w_*$ lies in the subspace spanned by the column vectors of $X$ and the low rank assumption directly implies that $\w_*$ is sparse with respect to the singular vectors of $X$.

We denote by $r$ the rank of matrix $X$.  The following theorem shows that the recovery error of Algorithm~\ref{alg:1} is small provided that (1) $X$ is of low rank (i.e., $r \ll \min(d, n)$), and (2) the number of random projections is sufficiently large.
\begin{thm} \label{thm:recovery}
Let $\w_*$ be the optimal solution to (\ref{eqn:primal}) and let $\wt$ be the solution recovered by Algorithm~\ref{alg:1}. For any $0 < \varepsilon \leq 1/2$, with a probability at least $1 - \delta$, we have
\[
    \|\wt-\w_* \|_2 \leq \frac{\varepsilon}{1 - \varepsilon} \|\w_*\|_2,
\]
provided
\[
m \geq \frac{(r+1) \log(2 r/\delta)}{c \varepsilon ^2},
\]
where constant $c$ is at least $1/4$.
\end{thm}

According to Theorem~\ref{thm:recovery}, the number of required random projections is $\Omega(r\log r)$. This is similar to a compressive sensing result if we view rank $r$ as the sparsity measure used in compressive sensing. Following the same arguments as compressive sensing, it may be possible to argue that $\Omega(r\log r)$ is optimal due to the result of the coupon collector's problem~\citep{mowani-1995-random}, although the rigorous analysis remains to be developed.

As a comparison, the following theorem shows that with a high probability, the naive solution $\wh$ given in (\ref{eqn:wh}) does not accurately recover the true optimal solution $\w_*$.
\begin{thm} \label{thm:bad-recovery}
For any $0 < \varepsilon \leq 1/3$, with a probability at least $1 - \exp(-(d - r)/32) - \exp(-m/32) - \delta$, we have
\[
    \|\wh - \w_*\|_2 \geq \frac{1}{2}\sqrt{\frac{d-r}{m}}\left( 1-  \frac{ \varepsilon \sqrt{2(1+\varepsilon)}}{1 - \varepsilon} \right)\|\w_*\|_2,
\]
provided the condition on $m$ in Theorem~\ref{thm:recovery} holds.
\end{thm}
As indicated by Theorem~\ref{thm:bad-recovery}, when $m$ is sufficiently larger than $r$ but significantly smaller than $d$, we have $\|\wh - \w_*\|_2 = \Omega(\sqrt{d/m}\|\w_*\|_2)$, indicating that $\wh$ does not approximate $\w_*$ well.

It is important to note that Theorem~\ref{thm:bad-recovery} does not contradict the previous results showing that the random projection based method could result in a small classification error if the data is linearly separable with a large margin. This is because, to decide whether $\wh$ carries a similar classification performance to $\w_*$, we need to measure the following term
\begin{eqnarray}
    \max\limits_{\x \in \mbox{span}(X), \|\x\|_2 \leq 1} \x^{\top}(\wh - \w_*). \label{eqn:measure}
\end{eqnarray}
Since $\|\wh - \w_*\|_2$ can also be written as
\[
    \|\wh - \w_*\|_2 = \max\limits_{\|\x\|_2\leq 1} \x^{\top}(\wh - \w_*),
\]
the quantity defined in (\ref{eqn:measure}) could be significantly smaller than $\|\wh - \w_*\|_2$ if the data matrix $X$ is of low rank. The following theorem quantifies this statement.

\begin{thm} \label{thm:recovery-1}
For any $0 < \varepsilon \leq 1/2$, with a probability at least $1 - \delta$, we have
\[
  \max\limits_{\x \in \mbox{span}(X), \|\x\|_2 \leq 1} \x^{\top}(\wh - \w_*) \leq \varepsilon\left(1 + \frac{1}{1 - \varepsilon}\right) \|\w_*\|_2,
\]
provided the condition on $m$ in Theorem~\ref{thm:recovery} holds.
\end{thm}
We note that Theorem~\ref{thm:recovery-1} directly implies the result of margin classification error for random projection~\citep{blue-2006-random,Margin_RP}. This is because when a data point $(\x_i, y_i)$ can be separated by $\w_*$ with a margin $\gamma$, i.e., $y_i\w_*^{\top}\x_i \geq \gamma\|\w_*\|$, it will be classified by $\wh$ with a margin at least $\gamma - \varepsilon\left(1 + \frac{1}{1 - \varepsilon}\right)$ provided $\gamma >\varepsilon \left(1 + \frac{1}{1 - \varepsilon}\right)$.

Based on Theorem~\ref{thm:recovery}, we now state the recovery result for the iterative method.
\begin{thm} \label{thm:recovery-iterative}
Let $\w_*$ be the optimal solution to (\ref{eqn:primal}) and let $\wt^T$ be the solution recovered by Algorithm~\ref{alg:2}. For any $0 < \varepsilon < 1/2$, with a probability at least $1 - \delta$, we have
\[
    \|\wt^T - \w_* \|_2 \leq \left(\frac{\varepsilon}{1 - \varepsilon}\right)^T \|\w_*\|_2,
\]
provided the condition on $m$ in Theorem~\ref{thm:recovery} holds.
\end{thm}
Notice that the number of random projection $m$ does not depend on the number of iterations $T$. That is because we only apply random projection once to reducing the dimensionality of the data. Theorem~\ref{thm:recovery-iterative} implies that we can recover the optimal solution with a relative error $\epsilon$, i.e., $\|\w_* - \wt^T\|_2 \leq \epsilon \|\w_*\|_2$, by using $\log_{(1 - \varepsilon)/\varepsilon} 1/\epsilon$ iterations.

\subsection{Full Rank}
If $X$ has full rank, we established the following theorem to bound the recovery error.
\begin{thm} \label{thm:recovery:new} Assume $\w_*$ lies in the subspace spanned by the first $k$ left singular vectors of $X$, and the loss $\ell(\cdot)$ is $\gamma$-smooth. For any $0 < \varepsilon \leq 1$, with a probability at least $1 - \delta$, we have
\[
    \|\wt-\w_* \|_2 \leq \frac{\varepsilon}{1 - \varepsilon} \left(1+ \frac{\sqrt{\lambda}}{\sqrt{\gamma}\sigma_k} \right) \|\w_*\|_2,
\]
provided
\[
m \geq \frac{\bar{r} \sigma_1^2}{c \varepsilon^2 (\lambda/\gamma+ \sigma_1^2) } \log \frac{2 d}{\delta},
\]
where $\sigma_i$ is the $i$-th singular value of $X$, $\bar{r}=\sum_{i=1}^d \frac{\sigma_i^2}{\lambda/{\gamma}+\sigma_i^2}$, and the constant $c$ is at least $1/32$.
\end{thm}
The above theorem implies the number of required random projections is $\Omega(\bar{r}\log d)$, which can be significantly smaller than $d$. The number $\bar{r}$ is closely related to the numerical $\sqrt{\frac{\lambda}{\gamma}}$-rank of $X$~\citep{Hansen:1999:RDI}. We say that $X$ has numerical $\nu$-rank $r_\nu$ if
\[
\sigma_{r_\nu}> \nu \geq \sigma_{r_\nu+1}.
\]
Using the notation of numerical rank, we have
\[
\bar{r} \leq r_{\sqrt{\lambda/\gamma}} + \sum_{i= r_{\sqrt{\lambda/\gamma}}+1}^d \frac{\sigma_i^2}{\lambda/{\gamma}+\sigma_i^2}.
\]
Thus, when the singular value $\sigma_i \ll \sqrt{\lambda/{\gamma}}$ for $i > r_{\sqrt{\lambda/\gamma}}$, which means that $X$ can be well approximated by a rank $r_{\sqrt{\lambda/\gamma}}$ matrix, we have $\bar{r} = O(r_{\sqrt{\lambda/\gamma}})$.

One remarking property of our approach is it enjoys a multiplicative bound even in the full rank case. Thus, as long as $\frac{\varepsilon}{1 - \varepsilon} \left(1+ \frac{\sqrt{\lambda}}{\sqrt{\gamma}\sigma_k} \right) < 1$, we can use Algorithm~\ref{alg:2} to reduce the reconstruction error exponentially over the iterations.  In contrast, the random projection based algorithm for SVD~\citep{Halko:2011:FSR},  although is able to accurately recover the eigen-space when the matrix is of low rank, it will result in a significant error in uncovering the subspace spanned by the top singular vectors when applied to matrices of full rank, and therefore is unable to recover the optimal solution $\w_*$ accurately.

Finally, we note that the assumption that the optimal solution lies in the subspace spanned by the top singular vectors has been used in kernel learning~\citep{Guo2012:389} and semi-supervised learning~\citep{ICML2012Ji_616}.
\section{Analysis} \label{sec:analysis}
Due to the limitation of space, we just provided the analysis for the low rank case.
Before presenting the analysis, we first establish some notations and facts. Let the SVD of $X$ be
\[
    X = U\Sigma V^{\top} = \sum_{i=1}^r \lambda_i \u_i \v_i^{\top},
\]
where $\Sigma=\mbox{diag}(\lambda_1,\ldots,\lambda_r)$, $U=(\u_1,\ldots, \u_r)$, $V=(\v_1,\ldots, \v_r)$, $\lambda_i$ is the $i$-th singular value of $X$,  $\u_i \in \R^d$ and $\v_i \in \R^n$ are the corresponding left and right singular vectors of $X$.  We define
\begin{equation}\label{eqn:gamma}
    \gb_* = \Sigma V^{\top} D(\y)\ab_*, \textrm{ and } \gbt = \Sigma V^{\top}D(\y) \abh_*.
\end{equation}
It is straightforward to show that
\[
\w_* =-\frac{1}{\lambda}U\Sigma V^{\top}D(\y)\ab_*= -\frac{1}{\lambda}U\gamma_*, \textrm{ and } \wt=-  \frac{1}{\lambda} U\Sigma V^{\top}  D(\y)\abh_*  =-\frac{1}{\lambda}U\gbt.
\]
Since $U$ is an orthogonal matrix, we have
\begin{equation}\label{eqn:norm}
\|\w_*\|_2 = \frac{1}{\lambda}\|\gb_*\|_2, \  \|\wt\|_2 = \frac{1}{\lambda}\|\gbt\|_2, \textrm{ and } \|\wt-\w_*\|_2 = \frac{1}{\lambda}\|\gbt-\gb_*\|_2.
\end{equation}
Let us define
\[
A=U^{\top}R\in\R^{r\times m}.
\]
It is easy to verify that $A$ is a Gaussian matrix of size $r\times m$.

\subsection{Proof of Theorem~\ref{thm:recovery}}

We first introduce the following concentration inequality for Gaussian random matrix, which serves the key to our analysis.
\begin{cor} \label{cor:gaussian}
Let $A \in \R^{r\times m}$ be a standard Gaussian random matrix. For any $0 < \varepsilon \leq 1/2$, with a probability at least $1 - \delta$, we have
\[
    \left\|\frac{1}{m}AA^{\top} - I \right\|_2\leq \varepsilon,
\]
provided
\[
m \geq \frac{(r+1) \log(2 r/\delta)}{c \varepsilon ^2},
\]
where $\|\cdot\|_2$ is the spectral norm of matrix and $c$ is a constant whose value is at least $1/4$.
\end{cor}

Define $L(\ab)$ and $\widehat{L}(\ab)$ as
\begin{eqnarray*}
    L(\ab) = -\sum_{i=1}^n \ell_*(\alpha_i) - \frac{1}{2\lambda} \ab^{\top}G\ab, \textrm{ and }
    \widehat{L}(\ab) = -\sum_{i=1}^n \ell_*(\alpha_i) - \frac{1}{2\lambda} \ab^{\top}\Gh\ab.
\end{eqnarray*}
Since $\abh_*$ maximizes $\widehat{L}(\ab)$ over the domain $\Omega^n$, we have
\begin{eqnarray}
\widehat{L}(\abh_*) \geq \widehat{L}(\ab_*) + \frac{1}{2\lambda} (\abh_* - \ab_*)^{\top}\Gh(\abh_* - \ab_*). \label{eqn:1}
\end{eqnarray}
Using the concaveness of $\widehat{L}(\ab)$, we have
\begin{eqnarray}
    & & \widehat{L}(\abh_*) +\frac{1}{2\lambda} (\abh_* - \ab_*)^{\top}\Gh(\abh_* - \ab_*) \nonumber \\
     & \leq  &\widehat{L}(\ab_*) + (\abh_* - \ab_*)^{\top}\left(\nabla \widehat{L}(\ab_*) - \nabla L(\ab_*) + \nabla L(\ab_*)\right) \nonumber \\
    & \leq & \widehat{L}(\ab_*) + \frac{1}{\lambda}(\abh_* - \ab_*)^{\top}(G-\Gh)\ab_*, \label{eqn:2}
\end{eqnarray}
where the last inequality follows from the fact that $(\abh_* - \ab_*)^{\top}\nabla L(\ab_*)\leq 0$ since $\ab_*$ maximizes $L(\ab)$ over the domain $\Omega^n$. Combining the inequalities in (\ref{eqn:1}) and (\ref{eqn:2}), we have
\[
\frac{1}{\lambda}(\abh_* - \ab_*)^{\top}(G-\Gh )\ab_* \geq \frac{1}{\lambda}(\abh_* - \ab_*)^{\top}\Gh(\abh_* - \ab_*).
\]
We rewrite $G$ and $\Gh$ as
\begin{eqnarray*}
    G &=& D(\y)V \Sigma U^{\top} U \Sigma V^{\top} D(\y) =  D(\y)V \Sigma \Sigma V^{\top} D(\y), \\
    \Gh &=& D(\y)V \Sigma U^{\top} \frac{RR^{\top}}{m} U \Sigma V^{\top} D(\y)= D(\y)V \Sigma  \frac{AA^{\top}}{m} \Sigma V^{\top} D(\y).
\end{eqnarray*}
Using the definitions of $\gb_*$ and $\gbt$ in (\ref{eqn:gamma}), we obtain
\begin{eqnarray}
    (\gbt - \gb_*)^{\top} \left( I- \frac{AA^{\top}}{m}\right)\gb_* \geq (\gbt - \gb_*)^{\top}\frac{AA^{\top}}{m}(\gbt - \gb_*). \label{eqn:9}
\end{eqnarray}
From Corollary~\ref{cor:gaussian}, with a probability at least $1 - \delta$, we have $\left \|I- \frac{1}{m} AA^{\top} \right \|_2 \leq \varepsilon$, under the given condition on $m$. Therefore, we obtain
\[
    (1 - \varepsilon)\|\gbt - \gb_*\|_2 \leq \varepsilon \|\gb_*\|_2.
\]
We complete the proof by using the equalities given in (\ref{eqn:norm}).

\subsection{Proof of Theorem~\ref{thm:recovery-iterative}} \label{sec:iterative}
At the $t$-th iteration, we consider the following optimization problem:
\begin{eqnarray}
    \min\limits_{\w \in \R^d} L^t(\w; X, \y) = \frac{\lambda}{2}\|\w + \wt^{t-1}\|_2^2 + \sum_{i=1}^n \ell\left(y_i(\w + \wt^{t-1})^{\top}\x_i\right), \label{eqn:primal-a-1}
\end{eqnarray}
where $\wt^{t-1}$ is the solution obtained from the $t-1$-th iteration. It is straightforward to show that $\Delta^{t}_* = \w_* - \wt^{t-1}$ is the optimal solution to (\ref{eqn:primal-a-1}).  Then we can use the dual random projection approach to recover $\Delta^{t}_*$ by $\widetilde{\Delta}^{t}$. If we can similarly show that
\[
\|\widetilde{\Delta}^{t} - \Delta_*^{t}\|_2 \leq \frac{\varepsilon}{1-\varepsilon}\|\Delta^{t}_*\|_2,
\]
then we update the recovered solution by $\wt^{t}= \wt^{t-1} + \widetilde{\Delta}^{t}$ and have
\[
\|\wt^{t} - \w_*\|_2 = \|\widetilde{\Delta}^{t} - \Delta_*^{t}\|_2 \leq \frac{\varepsilon}{1-\varepsilon}\|\Delta_*^{t}\|_2= \frac{\varepsilon}{1-\varepsilon}\|\wt^{t-1} -  \w_*\|_2.
\]
As a result, if we repeat the above process for $t=1,\ldots, T$, the recovery error of the last solution $\wt^{T}$ is upper bounded by
\[
\|\wt^{T} - \w_*\|_2 \leq  \left( \frac{\varepsilon}{1-\varepsilon}\right)^{T}\|\wt^0-\w_*\|_2 =  \left( \frac{\varepsilon}{1-\varepsilon}\right)^{T}\|\w_*\|_2,
\]
where we assume $\wt^0=\mathbf{0}$.

The remaining question is how to compute the $\widetilde\Delta^{t}$ using the dual random projection approach. In order to make the previous analysis remain valid for the recovered solution $\widetilde\Delta^{t}$ to the problem~(\ref{eqn:primal-a-1}), we need to write the primal optimization problem in the same form as in~(\ref{eqn:primal}). To this end, we first note that $\wt^{t-1}$ lies in the subspace spanned by $\x_1, \ldots, \x_n$, thus we write $\wt^{t-1}$ as
\[
    \wt^{t-1} = - \frac{1}{\lambda}XD(\y)  \abh^{t-1}_* =- \frac{1}{\lambda}\sum_{i=1}^n [\abh^{t-1}_*]_i y_i\x_i.
\]
Then, $L^t(\w; X, \y)$ can be written as
 \begin{eqnarray*}
    L^t(\w; X, \y) & = & \frac{\lambda}{2}\|\wt^{t-1}\|_2^2 + \frac{\lambda}{2}\|\w\|_2^2 + \lambda \w^{\top}\wt^{t-1} + \sum_{i=1}^n \ell \left(y_i\w^{\top}\x_i + y_i[\wt^{t-1}]^{\top}\x_i \right) \\
    & = & \frac{\lambda}{2}\|\wt^{t-1}\|_2^2 + \frac{\lambda}{2}\|\w\|_2^2 + \sum_{i=1}^n \ell \left(y_i\w^{\top}\x_i + y_i[\wt^{t-1}]^{\top}\x_i \right) - [\abh^{t-1}_*]_i y_i\w^{\top}\x_i \\
    & = & \frac{\lambda}{2}\|\wt^{t-1}\|_2^2 + \frac{\lambda}{2}\|\w\|_2^2 + \sum_{i=1}^n \ell^t_i \left(y_i\w^{\top}\x_i \right),
\end{eqnarray*}
where the new loss function $\ell^t_i(z), i=1, \ldots, n$ is defined as
\begin{eqnarray}
    \ell^t_i(z) = \ell \left(z + y_i[\wt^{t-1}]^{\top}\x_i \right) -  [\abh^{t-1}_*]_i z. \label{eqn:loss-1}
\end{eqnarray}
Therefore, $\Delta^{t}_*$ is the solution to the following problem:
\begin{equation*}
\min_{\w\in\R^d} \frac{\lambda}{2}\|\w\|_2^2 + \sum_{i=1}^n \ell^t_i \left(y_i\w^{\top}\x_i \right).
\end{equation*}

To apply the dual random projection approach to recover $\Delta^{t}_*$, we solve the following low-dimensional optimization problem:
\begin{equation*}
\min_{\z\in\R^m} \frac{\lambda}{2}\|\z\|_2^2 + \sum_{i=1}^n \ell^t_i \left(y_i\z^{\top}\xh_i\right),
\end{equation*}
where $\xh_i \in \R^m $ is the low-dimensional representation for example $\x_i \in \R^d$. The following derivation signifies that the above problem is equivalent to the problem in~(\ref{eqn:primal-a}).
\begin{eqnarray*}
& & \frac{\lambda}{2}\|\z\|_2^2+ \sum_{i=1}^{n} \ell^t_i\left(y_i\z^{\top}\xh_i \right) \\
& =  & \frac{\lambda}{2}\|\z\|_2^2 + \sum_{i=1}^{n} \ell \left(y_i\z^{\top}\xh_i + y_i[\wt^{t-1}]^{\top}\x_i \right) -  [\abh^{t-1}_*]_i y_i \z^{\top} \xh_i \\
& =  & \frac{\lambda}{2}\|\z\|_2^2 + \frac{\lambda}{\sqrt{m}} \z^{\top}(R^{\top}\wt^{t-1}) + \sum_{i=1}^{n} \ell \left(y_i\z^{\top}\xh_i + y_i[\wt^{t-1}]^{\top}\x_i \right) \\
& =  & \frac{\lambda}{2}\left\|\z + \frac{1}{\sqrt{m}} R^{\top}\wt^{t-1} \right\|_2^2 + \sum_{i=1}^{n} \ell \left(y_i\z^{\top}\xh_i + y_i[\wt^{t-1}]^{\top}\x_i \right) - \frac{\lambda}{2} \left\|\frac{1}{\sqrt{m}} R^{\top}\wt^{t-1}\right\|_2^2,
\end{eqnarray*}
where in the third line we use the fact that $\xh_i =  R^{\top} \x_i /\sqrt{m}$ and $\wt^{t-1}= - \sum_i[\abh^{t-1}_*]_iy_i\x_i/\lambda$.  Given the optimal solution $\z^{t}_*$ to the above problem, we can recover $\Delta^{t}_*$ by
\[
\widetilde{\Delta}^{t} = -\frac{1}{\lambda}XD(\y) \ah_*^{t},
\]
where $\ah_*^{t}$ is computed by
\[
[\ah_*^{t}]_i = \nabla \ell_i^t \left(y_i\xh_i^{\top}\z^t_* \right) = \nabla \ell \left(y_i\xh_i^{\top}\z^t_* + y_i[\wt^{t-1}]^{\top}\x_i \right) - [\abh^{t-1}_*]_i, \ i=1,\ldots,n.
\]
The updated solution $\wt^{t}$ is computed by
\[
\wt^{t} = \wt^{t-1} + \widetilde{\Delta}^t = - \frac{1}{\lambda}XD(\y) \left( \abh^{t-1}_*+ \ah_*^{t} \right) =  - \frac{1}{\lambda}XD(\y)\abh^{t}_*,
\]
where $[\abh^{t}_*]_{i} = [\abh^{t-1}_*]_i + [\ah_*^{t}]_i =\nabla\ell(y_i\xh_i^{\top}\z^t_* + y_i[\wt^{t-1}]^{\top}\x_i), i=1,\ldots,n$.
\section{Conclusion} \label{sec:conclusion}
In this paper, we consider the problem of recovering the optimal solution $\w_*$ to the original high-dimensional optimization problem using random projection. To this end, we propose to use the dual solution $\abh_*$ to the low-dimensional optimization problem to recover $\w_*$. Our analysis shows that with a high probability, the solution $\wt$ returned by our proposed method approximates the optimal solution $\w_*$ with small error.

\acks{This work is partially supported by Office of Navy Research (ONR Award
N00014-09-1-0663 and N000141210431).}

\bibliography{rand-proj}

\appendix
\section{Proof of Proposition~\ref{prop:1} and Proposition~\ref{prop:2}}
Since the two propositions can be proved similarly, we only present the proof of Proposition~\ref{prop:1}.
First, if $\ab_*$ is the optimal dual solution, by replacing $\ell(\cdot)$ in (\ref{eqn:primal}) with its conjugate form, the optimal primal solution can be solved by
\[
\w_*=\arg\min_{\w\in\R^d} \frac{\lambda}{2}\|\w\|_2^2 + \sum_{i=1}^n [\ab_*]_iy_i\x_i^{\top} \w.
\]
Setting the gradient with respect to $\w$ to zero, we obtain
\[
\w_*= - \frac{1}{\lambda}\sum_{i=1}^n[\ab_*]_iy_i\x_i=-\frac{1}{\lambda}XD(\y)\ab_*.
\]

Second, let's consider how to obtain the dual solution $\ab_*$ from the primal solution $\w_*$. Note that
\[
\ell(y_i\x_i^{\top}\w_*) = [\ab_*]_i \left (y_i\x_i^{\top}\w_* \right) - \ell_*\left([\ab_*]_i\right).
\]
By the Fenchel conjugate theory \citep{borwein2006convex,Cesa-Bianchi:2006:PLG:1137817}, we have $\ab_*$ satisfying
\[
[\ab_*]_i = \nabla \ell\left(y_i\x_i^{\top}\w_*\right), \ i=1,\ldots,n.
\]
\section{Proof of Corollary~\ref{cor:gaussian}}
In the proof, we make use of the recent development in tail bounds for the eigenvalues of a sum of random matrices~\citep{gittens-2011-tail,2012:zhu}.
\begin{thm}[Theorem 1~\citep{2012:zhu}] \label{thm:variance}  Let $\{\bm \xi_j:j=1,\ldots,n\}$ be i.i.d.~samples drawn from a multivariate Gaussian distribution $\N(0, C)$, where $C \in \R^{d \times d}$. Define
\[
  \widehat{C}_n  = \frac{1}{n} \sum_{j=1}^n \bm \xi_j \bm \xi_j^{\top}.
\]
We denote the trace of $X$ by $\tr(X)$, and the spectral norm of $X$ by $\|X\|$. Then, for any $\theta \geq 0$
\[
    \Pr\left\{ \left \|\widehat{C}_n  - C \right \| \geq \left( \sqrt{\frac{2 \theta (k+1)}{n}} + \frac{2 \theta k}{n} \right) \|C\| \right\} \leq 2d \exp(-\theta),
\]
where $k=\tr(C)/\|C\|$.
\end{thm}
We write $A = (\bm \xi_1, \ldots, \bm \xi_m)$, where $\bm \xi_i \in \R^r$ is i.i.d.~sampled from the Gaussian distribution $\N(0, I)$, and write $AA^{\top}/m$ as
\[
 \frac{1}{m}AA^{\top} = \frac{1}{m}\sum_{i=1}^m \eb_i\eb_i^{\top}.
\]
Following Theorem~\ref{thm:variance}, we have, with a probability at least $1-2r\exp(-\theta)$
\[
\left\|\frac{1}{m}AA^{\top} - I\right \| \leq \sqrt{\frac{2 \theta (r+1)}{m}} + \frac{2 \theta r}{m}.
\]
By setting $2r\exp(-\theta)=\delta$, we have, with a probability at least $1-\delta$
\[
\left\|\frac{1}{m}AA^{\top}  - I \right \| \leq \sqrt{\frac{2 (r+1) \log(2r/\delta) }{m}} + \frac{2 r \log(2r/\delta)}{m}  \leq   \varepsilon \sqrt{2 c }+ 2 c\varepsilon^2 \leq (\sqrt{2 c}+ c) \varepsilon \leq \varepsilon,
\]
provided
\[
m \geq \frac{(r+1) \log(2 r/\delta)}{c \varepsilon ^2},  \ \varepsilon \leq \frac{1}{2}, \textrm{ and } c= 2-\sqrt{3} \geq \frac{1}{4}.
\]
\section{Proof of Theorem~\ref{thm:rand-measure}}
 Before presenting our analysis, we first state a version of Johnson-Lindenstrauss theorem that is useful to our analysis.
\begin{thm}[Theorem 2~\citep{blue-2006-random}]\label{thm:jl} Let $\x \in \R^d$, and $\xh = R^{\top} \x/\sqrt{m}$, where $R \in \R^{d \times m}$ is a random matrix whose entries are chosen independently from $\N(0, 1)$. Then
\[
    \Pr\left\{(1 - \varepsilon)\|\x\|_2^2 \leq \|\xh\|_2^2 \leq (1 + \varepsilon)\|\x\|_2^2 \right\} \geq 1 - 2\exp\left(-\frac{m}{4}(\varepsilon^2 - \varepsilon^3) \right).
\]
\end{thm}
According to (\ref{eqn:9}) in the proof of Theorem~\ref{thm:recovery}, we have
\[
    (\gbt - \gb_*)^{\top}\left(I- \frac{AA^{\top}}{m}\right)\gb_* \geq (\gbt - \gb_*)^{\top}\frac{AA^{\top}}{m}(\gbt - \gb_*).
\]
Notice that
\begin{eqnarray*}
\sqrt{m}\z_* &=&  -\frac{1}{\lambda}R^{\top}XD(\y)\abh_*=  -\frac{1}{\lambda}R^{\top} U\Sigma V^{\top} D(\y)\abh_*=   -\frac{1}{\lambda}R^{\top} U  \gbt = -\frac{1}{\lambda} A^{\top} \gbt, \\
R^{\top}\w_* &=& -\frac{1}{\lambda}R^{\top} U\gamma_* =  -\frac{1}{\lambda}A^{\top} \gamma_*.
\end{eqnarray*}
Then, we have
\[
    \frac{\lambda^2}{m}\|\sqrt{m}\z_* - R^{\top}\w_*\|_2^2 \leq (\gbt - \gb_*)^{\top}\left(I- \frac{AA^{\top}}{m}\right)\gb_*.
\]
Using Corollary~\ref{cor:gaussian}, with a probability at least $1 - \delta$, we have
\[
    \frac{1}{m}\|\sqrt{m}\z_* - R^{\top}\w_*\|_2^2 \leq \varepsilon\|\w_*\|_2\|\wt - \w_*\|_2.
\]
Following Theorem~\ref{thm:recovery}, with a probability at least $1 - \delta$, we have
\begin{eqnarray}
    \frac{1}{m}\|\sqrt{m}\z_*- R^{\top}\w_*\|_2^2 \leq \frac{\varepsilon^2}{1 - \varepsilon}\|\w_*\|_2^2. \label{eqn:10}
\end{eqnarray}
To replace $\w_*$ on R. H. S. of the above inequality with $R^{\top}\w_*$, we make use of Theorem~\ref{thm:jl}. With a probability at least $1 - \exp(-(\tau^2 - \tau^3)m/4)$, we have
\[
(1-\tau) \|\w_*\|_2^2 \leq \frac{1}{m}\|R^{\top}\w_*\|_2^2.
\]
By choosing $\tau = 1/2$, we have, with a probability at least $1 - \exp(-m/32)$
\begin{eqnarray}
 \frac{1}{2}\|\w_*\|_2^2 \leq \frac{1}{m}\|R^{\top}\w_*\|_2^2.  \label{eqn:11}
\end{eqnarray}
We complete the proof by combining the two inequalities in (\ref{eqn:10}) and (\ref{eqn:11}).
\section{Proof of Theorem~\ref{thm:bad-recovery}}
As discussed before, the key reason for the large difference between $\wh$ and $\w_*$ is because they do not lie in the same subspace: $\w_*$ lies in the subspace spanned by the columns in $U$ while $\wh$ lies in the subspace spanned by the column vectors in a random matrix.

In the subspace orthogonal to $\u_1, \ldots, \u_r$, we randomly choose a subset of $d-r$ orthogonal bases, denoted by $\u_{r+1}, \ldots, \u_{d}$. Let $U_{\perp} = (\u_{r+1}, \ldots, \u_d)$.
Since
\[
    \|\wh-\w_*\|_2 = \max\limits_{\|\x\|_2 \leq 1} \x^{\top}(\wh-\w_*),
\]
to facilitate our analysis, we restrict the choice of $\x$ to the subspace spanned by $\u_{r+1}, \ldots, \u_d$ and have
\[
    \|\wh-\w_*\|_2 \geq \max\limits_{\x \in \mbox{span}(\u_{r+1}, \ldots, \u_d), \|\x\|_2 \leq 1} \x^{\top}\wh,
\]
where we use the fact $\w_* \perp \mbox{span}(\u_{r+1}, \ldots, \u_d)$. From Proposition~\ref{prop:2}, we can express $\wh$ as
\[
\wh = \frac{1}{\sqrt{m}} R\z_* = -\frac{1}{m \lambda} R R^{\top}XD(\y)\abh_*=  -\frac{1}{m\lambda} R R^{\top}  U\Sigma V^{\top} D(\y)\abh_* = -\frac{1}{m\lambda} R R^{\top} U \gbt,
\]
where $\gbt$ is defined in (\ref{eqn:gamma}). Write $\x$ as $\x = U_{\perp} \a$, where $\a \in \R^{d - r}$. Define
\[
\Lambda = U^{\top}_{\perp}R \in \R^{(d - r)\times m}.
\]
As a result, we bound $\|\w_* - \wh\|_2$ by
\begin{eqnarray} \label{eqn:b}
\max\limits_{\x \in \mbox{span}(\u_{r+1}, \ldots, \u_m), \|\x\|_2 \leq 1 } \x^{\top}\wh = \max\limits_{\|\a\|_2 \leq 1} \frac{1}{m \lambda}\a^{\top}U^{\top}_{\perp}RR^{\top}U\gbt = \frac{1}{m \lambda}\|\Lambda A^{\top}\gbt\|_2.
\end{eqnarray}

It is easy to verify that $A$ and $\Lambda$ are two independent Gaussian random matrices. Therefore, we can fix the vector $A^{\top}\gbt$ and estimate how the random matrix $\Lambda$ affect the norm of vector $A^{\top}\gbt$. According to Theorem~\ref{thm:jl} (i.e., Johnson-Lindenstrauss theorem), for a fixed vector $A^{\top}\gbt$, with a probability at least $1 - \exp(-(d-r)/32)$
\begin{equation} \label{eqn:b1}
    \frac{1}{\sqrt{d- r}}\|\Lambda A^{\top}\gbt\|_2 \geq \frac{1}{\sqrt{2}}\|A^{\top}\gbt\|_2.
\end{equation}

We now bound $\|A^{\top}\gbt\|_2$. Note that we cannot directly apply Theorem~\ref{thm:jl} to bound the norm of $A^{\top}\gbt$ because $\gbt$ is a random variable depending on the random matrix $A$. To decouple the dependence between $A$ and $\gbt$, we expand $\|A^{\top}\gbt\|_2$ as
\begin{equation} \label{eqn:b2}
\|A^{\top}\gbt\|_2 \geq \|A^{\top}\gb_*\|_2 - \|A^{\top}(\gb_* - \gbt)\|_2,
\end{equation}
where $\gb_*$ is defined in (\ref{eqn:gamma}). We bound the two terms on the right side of the inequality in (\ref{eqn:b2}) separately. Using Theorem~\ref{thm:jl}, with a probability at least $1 - \exp(-m/32)$, we bound $\|A^{\top}\gb_*\|$ by
\begin{equation} \label{eqn:b3}
    \frac{1}{\sqrt{m}}\|A^{\top}\gb^*\|_2 \geq \frac{1}{\sqrt{2}}\|\gb_*\|_2 = \frac{\lambda}{\sqrt{2}} \|\w_*\|_2.
\end{equation}
To bound the second term $\|A^{\top}(\gb_* - \gbt)\|$, with a probability at least $1-\delta$,  we have
\[
    \frac{1}{\sqrt{m}}\|A^{\top}(\gb_* - \gbt)\|_2 \leq \sqrt{\lambda_{\max}(AA^{\top}/m)}\|\gb_* - \gbt\|_2\leq \sqrt{1+\varepsilon}\lambda\|\w_*-\wt\|_2,
\]
where we use the result in Corollary~\ref{cor:gaussian}. According to Theorem~\ref{thm:recovery}, we have
\[
    \|\w_* - \wt\|_2 \leq \frac{\varepsilon}{1 - \varepsilon}\|\w_*\|_2.
\]
As a result, with a probability at least $1-\delta$, we have
\begin{equation} \label{eqn:b4}
\frac{1}{\sqrt{m}}\|A^{\top}(\gb_* - \gbt)\|_2 \leq \lambda \sqrt{1 + \varepsilon}\frac{\varepsilon}{1 - \varepsilon}\|\w_*\|_2.
\end{equation}
We complete the proof by putting together (\ref{eqn:b}), (\ref{eqn:b1}), (\ref{eqn:b2}), (\ref{eqn:b3}), and (\ref{eqn:b4}).
\section{Proof of Theorem~\ref{thm:recovery-1}}
It is straightforward to check that
\[
\wh = \frac{RR^{\top}}{m} \wt.
\]
Therefore,
\begin{eqnarray*}
& & \max\limits_{\x \in \mbox{span}(X), \|\x\|_2 \leq 1} \x^{\top}(\wh - \w_*) \\
 &\leq& \|\wt-\w_* \|_2 + \max\limits_{\|\x\|_2 \leq 1, \x \in   \mbox{span}(X)} \x^{\top}(\wh-\wt) \\
& = & \|\wt-\w_*\|_2 + \max\limits_{\|\a\|_2 \leq 1} \a^{\top}\left(\frac{1}{m} U^{\top}RR^{\top}U -I\right) \gbt/\lambda\\
& \leq& \|\wt -\w_*\|_2 + \lambda_{\max}\left(\frac{1}{m} U^{\top}RR^{\top}U -I \right) \|\wt\|_2 \\
& \leq & \|\wt-\w_*\|_2 + \lambda_{\max}\left(\frac{1}{m} AA^{\top} -I\right) \|\w_*\|,
\end{eqnarray*}
where in the fourth line we use the fact $\|\wt\|_2 = \|\gbt\|_2/\lambda$. Using Corollary~\ref{cor:gaussian}, we have, with a probability at least $1 - \delta$
\[
    \lambda_{\max}\left(\frac{1}{m} AA^{\top} -I\right) \leq \varepsilon.
\]
We complete the proof by using the bound for $\|\wt- \w_* \|_2$ stated in Theorem~\ref{thm:recovery}.
\section{Proof of Theorem~\ref{thm:recovery:new}}
Define $L(\ab)$ and $\widehat{L}(\ab)$ as
\begin{eqnarray*}
    L(\ab) = -\sum_{i=1}^n \ell_*(\alpha_i) - \frac{1}{2\lambda} \ab^{\top}G\ab, \textrm{ and }
    \widehat{L}(\ab) = -\sum_{i=1}^n \ell_*(\alpha_i) - \frac{1}{2\lambda} \ab^{\top}\Gh\ab.
\end{eqnarray*}
Since $\ell(\cdot)$ is $\gamma$-smooth, and thus $\ell_*(\cdot)$ is $\frac{1}{\gamma}$-strongly convex. Define
\[
g_*(\alpha) = \ell_*(\alpha) - \frac{1}{2\gamma}\alpha^2, \ H = G + \frac{\lambda}{\gamma} I, \textrm{ and }\Hh = \Gh + \frac{\lambda}{\gamma} I.
\]
Evidently, $g_*(\alpha)$ is still a convex function. We write $\L(\ab)$ and $\widehat{L}(\ab)$ as
\[
    L(\ab) = -\sum_{i=1}^n g_*(\alpha_i) - \frac{1}{2\lambda} \ab^{\top}H\ab, \textrm{ and }
    \widehat{L}(\ab) = -\sum_{i=1}^n g_*(\alpha_i) - \frac{1}{2\lambda} \ab^{\top}\Hh\ab.
\]

Since $\abh_*$ maximizes $\widehat{L}(\ab)$ over the domain $\Omega^n$, we have
\begin{eqnarray}
\widehat{L}(\abh_*) \geq \widehat{L}(\ab_*) + \frac{1}{2\lambda} (\abh_* - \ab_*)^{\top}\Hh(\abh_* - \ab_*). \label{eqn:1:n}
\end{eqnarray}
Using the concaveness of $\widehat{L}(\ab)$, we have
\begin{eqnarray}
    & & \widehat{L}(\abh_*) +\frac{1}{2\lambda} (\abh_* - \ab_*)^{\top}\Hh(\abh_* - \ab_*) \nonumber \\
     & \leq & \widehat{L}(\ab_*) + (\abh_* - \ab_*)^{\top}\left(\nabla \widehat{L}(\ab_*) - \nabla L(\ab_*) + \nabla L(\ab_*)\right) \nonumber \\
    & \leq & \widehat{L}(\ab_*) + \frac{1}{\lambda}(\abh_* - \ab_*)^{\top}(H-\Hh)\ab_*, \label{eqn:2:n}
\end{eqnarray}
where the last inequality follows from the fact that $(\abh_* - \ab_*)^{\top}\nabla L(\ab_*)\leq 0$ since $\ab_*$ maximizes $L(\ab)$ over the domain $\Omega^n$. Combining the inequalities in (\ref{eqn:1:n}) and (\ref{eqn:2:n}), we have
\begin{eqnarray}
\frac{1}{\lambda}(\abh_* - \ab_*)^{\top}(H-\Hh )\ab_* \geq \frac{1}{\lambda}(\abh_* - \ab_*)^{\top}\Hh(\abh_* - \ab_*). \label{eqn:3:n}
\end{eqnarray}
Define $K = H^{-1/2} \Hh H^{-1/2}$. We rewrite the bound in (\ref{eqn:3:n}) as
\[
(\abh_* - \ab_*)^{\top}H^{1/2}\left(I - K\right)H^{1/2}\ab_*  \geq (\abh_* - \ab_*)^{\top} H^{1/2} K H^{1/2}(\abh_* - \ab_*).
\]

To bound the spectral norm of $K$, we have the following lemma.
\begin{lemma}\label{lemma:spectral} With a probability at least $1-\delta$, we have
\[
(1-\varepsilon) I \preceq  K \preceq (1+ \varepsilon) I,
\]
provided the condition on $m$ in Theorem~\ref{thm:recovery:new} holds.
\end{lemma}
\begin{proof}[Lemma~\ref{lemma:spectral}]
Let the SVD of $X$ be
\[
    X = U\Sigma V^{\top} = \sum_{i=1}^d \sigma_i \u_i \v_i^{\top},
\]
where $\Sigma=\mbox{diag}(\sigma_1,\ldots,\sigma_d)$, $U=(\u_1,\ldots, \u_d)$, $V=(\v_1,\ldots, \v_d)$, $\sigma_i$ is the $i$-th singular value of $X$,  $\u_i \in \R^d$ and $\v_i \in \R^n$ are the corresponding left and right singular vectors of $X$. Since $y_i \in \{-1, +1\}$, it is straightforward to check that the SVD of $X D(\y)$ is given by
\[
    X D(\y) = U\Sigma [D(\y)V]^{\top} = \sum_{i=1}^d \sigma_i \u_i [D(\y) \v_i]^{\top},
\]
and the eigen decomposition of $G= D(\y)X^{\top}XD(\y)$ is
\[
G= \sum_{i=1}^d \sigma_i^2 [D(\y) \v_i] [D(\y) \v_i]^{\top}.
\]
Following the Corollary \ref{cor:2} in Appendix~\ref{sec:final}, we obtain Lemma~\ref{lemma:spectral}.
\end{proof}

From Lemma~\ref{lemma:spectral}, we have, with a probability at least $1 - \delta$,
\[
    \|I - K\|_2 \leq \varepsilon,
\]
and therefore
\begin{equation} \label{eqn:bound:1:n}
 \varepsilon \left \|H^{1/2}\ab_*  \right\|_2 \geq  (1-\varepsilon) \left \| H^{1/2}  (\abh_* - \ab_*)\right\|_2 \geq (1-\varepsilon) \left \| G^{1/2}  (\abh_* - \ab_*)\right\|_2.
\end{equation}
Since
\[
\w_*  =-\frac{1}{\lambda}XD(\y)\ab_*  = -\frac{1}{\lambda} \sum_{i=1}^d  \sigma_i \left([D(\y)\v_i]^{\top}\ab_*\right) \u_i,
\]
the assumption that $\w_*$ lies in the subspace spanned by $\{\u_1,\ldots,\u_k\}$ implies $\ab_*$ lies in the subspace spanned by $\{D(\y) \v_1,\ldots, D(\y) \v_k\}$. Then, we have
\[
\ab_*^{\top} G \ab_* = \ab_*^{\top} G_k \ab_* \geq \sigma_k^2 \ab_*^T \ab_*,
\]
where $G_k$ is the rank-$k$ best approximation of $G$ and $\sigma_k$ is the $k$-th singular value of $X$. As a result, we conclude
\begin{equation} \label{eqn:bound:2:n}
\|H^{1/2}\ab_*\|_2^2  =  \ab_*^{\top}\left(G + \frac{\lambda}{\gamma} I \right)\ab_*  \leq \left(1 + \frac{\lambda}{\gamma\sigma_{k}^2}\right)  \ab_*^{\top} G \ab_*= \left(1 + \frac{\lambda}{\gamma\sigma_{k}^2}\right) \|G^{1/2}\ab_*\|_2^2.
\end{equation}
Combining (\ref{eqn:bound:1:n}) and (\ref{eqn:bound:2:n}), we have, with a probability at least $1 - \delta$
\[
\epsilon \sqrt{1 + \frac{\lambda}{\gamma\sigma_k^2}}\|G^{1/2}\ab_*\|_2 \geq (1 -\epsilon)\|G^{1/2}(\ab_* - \abh_*)\|_2.
\]
We complete the proof by using the relationship between $\w_*$, $\wt$ and $\ab_*$, $\abh_*$.
\section{A matrix concentration inequality} \label{sec:final}
\begin{thm}
  \label{thm:1}
  Let $C=\mbox{diag}(c_1,\ldots,c_p)$ and $S=\mbox{diag}(s_1,\ldots,s_p)$ be $p\times p$ diagonal matrices, where $c_i\neq 0$ and
  $c_i^2+s_i^2=1$ for all $i$.  Let $R$ be a Gaussian random matrix of
  size $p\times n$.  Let $M= C^2+\frac{1}{n} S R R^\top S$ and $r
  = \sum_i s_i^2$.
  \[
  \begin{split}
    \label{eq:1}
    \Pr(\lambda_1(M) \geq 1+t) & \leq q \cdot \exp\left(
    -\frac{cn t^2}{\max_i(s_i^2)r} \right), \\
    \Pr(\lambda_p(M) \leq 1-t) & \leq
    q\cdot \exp \left(-\frac{cn t^2}{\max_i(s_i^2)r} \right),
  \end{split}
  \]
  where the constant $c$ is at least $1/32$, and $q$ is the rank of $S$.
\end{thm}
\begin{proof}[Theorem~\ref{thm:1}]
The proof is similar to Theorems 5.3 and 7.1 of \citep{gittens-2011-tail}, expect for adding a bias matrix.
Let $g(\theta)=\frac{\theta^2}{2(1-\theta)}$. We have
  \[
  \begin{split}
    &\Pr\left\{\lambda_1(C^2+ \frac{1}{n}S RR^\top S) \geq
      1+t \right\}\\
    \leq & \inf_{\theta>0} \tr\exp\left\{
      \theta\left( C^2+\frac{1}{n} \E\left[S RR^\top S\right]  -
        (1+t)I
        \right) + \frac{1}{n}g(\theta) \E\left[ (S RR^\top S)^2\right]
    \right\} \\
    \leq & \inf_{\theta>0} \tr\exp\left\{
      -\theta t  + 8g(\theta)\tr (S^2)S^2
    \right\} \\
    \leq & \inf_{\theta>0} q \exp\left\{ - \theta t
      + 8 r g(\theta) \right\}
    \leq q \exp\left( -\frac{nt^2}{32r}\right), \\
    & \Pr\left\{\lambda_q(C^2+ \frac{1}{n}S RR^\top S) \leq 1-t\right\}\\
    \leq & \inf_{\theta>0} \tr\exp\left\{
      -\theta\left( C^2+\frac{1}{n} \E\left[ S RR^\top S \right]  -
        (1-t)I
        \right) + \frac{1}{n}g(\theta) \E\left[(S RR^\top S)^2\right]
    \right\} \\
    \leq  & \inf_{\theta>0} q \exp\left\{-\theta t
      +8r g(\theta) \right\}
    \leq q \exp\left( -\frac{nt^2}{32r}\right).
  \end{split}
  \]
\end{proof}
\begin{cor}
  \label{cor:2}
  Let $A$ be a given matrix of size $m\times p$, and $R$ be a Gaussian
  random matrix of size $p\times n$.  Let $\lambda$ be a positive
  constant, $\sigma_i^2=\lambda_i(A^\top A)$ , and $r= \sum_i
  \frac{\sigma^2_i}{\lambda +\sigma^2_i}$.  Let $K=\lambda I_m +
  AA^\top$, $\tilde{K}=\lambda I_m + \frac{1}{n}
  ARR^\top A^\top$, and $\tilde{I} = K^{-1/2} \tilde{K}
  K^{-1/2}$. If $n\geq \frac{r \sigma^2_1}{c t^2(\lambda+\sigma^2_1)}
  \log \frac{2p}{\delta}$, then with probability at least $1-\delta$,
  \begin{equation}
    \label{eq:2}
    (1-t) I_m \preceq \tilde{I}  \preceq (1+t) I_m.
  \end{equation}
\end{cor}
\begin{proof}[Corollary~\ref{cor:2}] Let $s_i^2=\frac{\sigma^2_i}{\lambda + \sigma^2_i}$, and
  $c_i^2=1-s_i^2$.
  By SVD and Theorem~\ref{thm:1}, we have
  $\Pr(1-t \leq \lambda_i(\tilde{I}) \leq 1+t) \geq 1-\delta$.
\end{proof}
\end{document}